\pdfoutput=1
\documentclass[11pt]{article}

\usepackage[final]{acl}

\usepackage{times}
\usepackage{latexsym}

\usepackage[T1]{fontenc}
\usepackage[utf8]{inputenc}

\usepackage{microtype}

\usepackage{inconsolata}
\usepackage{amsthm}
\newtheorem{proposition}{Proposition}
\usepackage{graphicx}
\usepackage{fontawesome5}
\usepackage{xcolor}
\usepackage{mdframed}
\usepackage[most]{tcolorbox}
\usepackage{subcaption} % 用于插入子图
\usepackage{hyperref}
\usepackage{booktabs}
\usepackage{algorithm}
\usepackage[noend]{algorithmic}
\usepackage{amssymb}
\usepackage{CJKutf8}
\usepackage{colortbl}
\usepackage{multirow} 
\usepackage{amsmath}
\title{Efficient Provably Secure Linguistic Steganography via Range Coding}

\author{Ruiyi Yan \and Yugo Murawaki \\
        Graduate School of Informatics, Kyoto University \\ \texttt{ruiyi@nlp.ist.i.kyoto-u.ac.jp}, \texttt{murawaki@i.kyoto-u.ac.jp}}

\begin{document}
\maketitle
\begin{abstract}
%background
Linguistic steganography involves embedding secret messages within seemingly innocuous texts to enable covert communication. Provable security, which is a long-standing goal and key motivation, has been extended to language-model-based steganography.
%Challenge
Previous provably secure approaches have achieved perfect imperceptibility, measured by zero Kullback-Leibler (KL) divergence, but at the expense of embedding capacity.
%Methods
In this paper, we attempt to directly use a classic entropy coding method (\textbf{range coding}) to achieve secure steganography, and then propose an efficient and provably secure linguistic steganographic method with a rotation mechanism.
%Experiments
Experiments across various language models show that our method achieves around 100\% entropy utilization (embedding efficiency) for embedding capacity, outperforming the existing baseline methods. Moreover, it achieves high embedding speeds (up to 1554.66 bits/s on GPT-2).
The code is available at \faGithub~\href{https://github.com/ryehr/RRC\_steganography}{github.com/ryehr/RRC\_steganography}.
\end{abstract}

\section{Introduction}

Linguistic steganography, as a promising field in safeguarding information, refers to the art of concealing messages within texts. 
With rapid advancements in large language models (LLM)~\cite{NEURIPS2020_1457c0d6,achiam2023gpt, Claude3}, LM-based steganography methods~\cite{ziegler-etal-2019-neural,10.1145/3664647.3680562, yan2025comprehensive} have dominated linguistic steganography, as leveraging LMs can create flexible text content, diverse genres, and consistent contexts, enabling high embedding capacity. 
Figure~\ref{fig: linguistic_steganography} illustrates how a sender (Alice) and a receiver (Bob) communicate using linguistic steganography.

\begin{figure}[!t]
 \centering
 \includegraphics[width=\columnwidth]{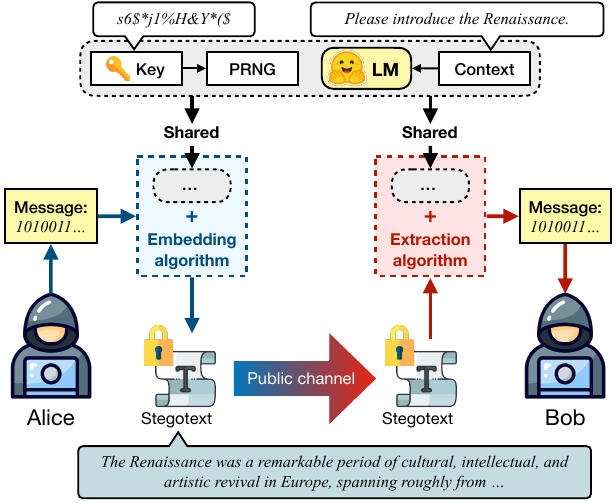} 
 \caption{A schematic diagram of linguistic steganography, where PRNG refers to a pseudo-random number generator for controlling randomness and reproducibility. Alice embeds the secret message into a steganographic text (stegotext), and Bob extracts the secret message from the received stegotext.}
 \label{fig: linguistic_steganography}
\end{figure}

Simmons’s ``Prisoners’ Problem''~\cite{Simmons1984} illustrates a steganographic scenario, where Alice and Bob (steganographers) are trying to hatch an escape plan, and they can only communicate under the scrutiny of a warden, Eve (a steganalyzer), who can block communication if any suspicious content is detected. To evade detection, they embed their secret message into an innocent-looking carrier (steganographic content).
Intuitively, steganographic content is expected to closely resemble normal content, leading to the concept of steganographic security. 
% This notion was first formalized by~\citet{10.1007/3-540-49380-8_21} using the Kullback–Leibler (KL) divergence between the cover distribution $P_c$ and the steganographic distribution $P_s$.
However, incorporating steganographic algorithms into the language model’s prediction and sampling processes often introduces distributional distortions~\cite{10.1145/2482513.2482514, 8470163}.
To address this challenge, recent work has explored approaches aimed at achieving \textit{provable security}~\cite{10.1007/3-540-45708-9_6, 4663056} in steganography.
% \footnote{Related work is introduced in Appendix~\ref{sec: Related Work} in detail.}

Despite the progress of provably secure approaches, each method has its own limitations. ADG~\cite{zhang-etal-2021-provably} fails to strictly preserve the original probability distribution by grouping candidate tokens at each generative step. Meteor~\cite{10.1145/3460120.3484550}, which is based on arithmetic coding (AC)~\cite{ziegler-etal-2019-neural}, inevitably distorts the original distribution when encoding intervals. 
Although iMEC~\cite{witt2023perfectly}, Discop~\cite{10179287}, and SparSamp~\cite{wang2025sparsampefficientprovablysecure} maintain the original probability distribution, the first two suffer from limited embedding capacity and slow embedding speeds. 
SparSamp, the current state-of-the-art method, still falls short of achieving full embedding efficiency (i.e., 100\% entropy utilization).
All of these methods process the secret message as discrete bits, which can result in losses in embedding efficiency or imperceptibility when incorporating some mechanisms for unique extraction.

These limitations motivate us to revisit a classical entropy-coding technique, range coding~\citep[RC;][]{martin1979range} which encodes messages in decimal form rather than in bits, and explore whether it can achieve (i) preservation of the original probability distribution, (ii) full embedding efficiency, and (iii) high embedding speed.\footnote{When all operations are carried out in decimals, the original distribution at each generative step is rescaled to a new range, without distortion caused by constructing discrete bits.} 
% Coding done with digits instead of with bits makes RC well-suited for preserving the original probability distribution and achieving higher embedding speed.
% In this paper, instead of designing sophisticated tricks, we turn our attention to a classic entropy coding method, range coding~\citep[RC;][]{martin1979range}. 
% RC is very similar to AC in data compression, except that the coding is done with digits in any base, instead of bits.
% It is convenient to take advantage of this feature to preserve the original probabilities and a higher embedding speed, because if all processes are conducted in decimals without any binary encoding, the original probability distribution at each generative step are only rescaled into new ranges without ratio distortion.
% Besides, RC, as an entropy coding method just like AC, naturally enables RC-based steganography to take full advantage of entropy to achieve ideal embedding capacity.
% In this paper, rather than combining complex algorithms, we turn our attention to a classic entropy coding method: range coding~\citep[RC;][]{martin1979range}.
% RC is closely related to arithmetic coding (AC) in the context of data compression, but with a key difference: \textit{it performs encoding using digits in any base, rather than restricting to bits}.
% The key contributions of this work are as follows: 

% Furthermore, as an entropy coding method akin to arithmetic coding (AC), RC inherently allows RC-based steganography to fully utilize the entropy, thereby achieving ideal embedding capacity. 
% 

We begin by proposing a vanilla RC steganography (Section~\ref{sec: Vanilla Range-Coding Steganography}). To the best of our knowledge,  it is the first attempt to embed a secret message \textbf{entirely in decimal form}.
However, we discover its security issues: (1) distortion on probability distribution and (2) randomness reuse.

To address these issues, we present \textbf{rotation range-coding (RRC)} steganography, which incorporates a rotation mechanism (Section~\ref{sec: Rotation Range-Coding Steganography}). This mechanism ensures zero KL divergence at every generative step and prevents reuse of randomness. The method is \textbf{training-free} and \textbf{plug-and-play}.

In addition, we provide theoretical analysis and proofs regarding zero KL divergence and computational security, supporting the \textbf{provable security} of our RRC steganography (Section~\ref{sec: Analysis of RRC Steganography}).
Additionally, the approximate 100\% entropy utilization for embedding capacity is analyzed and empirically validated.
The key advantage of RC-based steganography is \textit{ensuring unique extraction internally, without requiring any trade-offs or external restrictions.}

Experimental results in various language models demonstrate that RRC steganography consistently achieves \textbf{the highest embedding efficiency} (i.e., entropy utilization) and \textbf{great embedding speed} (up to 1554.66 bits/s in GPT-2) compared to all provably secure baseline methods (Section~\ref{sec: Experiments}).
Experiments also show that RRC steganography has strong scalability and anti-steganalysis capacity.

\section{Background and Preliminaries}
\subsection{Language Model Basics}
A language model (LM) has a vocabulary $\mathcal{V}$, a set of tokens.
% The size of typical vocabularies ($|\mathcal{V}|$) is greater than $50{,}000$ tokens~\cite{radford2019language}.
Consider a sequence of LM-generated $T$ tokens $\{s^{(t)}\} \in \mathcal{V}^T$. 
Tokens with negative indices, $[s^{(-N_p)},\dots,s^{(-1)}]$, represent a \textit{prompt} of length $N_p$ and $[s^{(0)},\dots,s^{(T-1)}]$ are tokens generated by an LM in response to the prompt.

The next token prediction by an LM at position $t$, is a function whose input is a sequence of known tokens $[s^{(-N_p)},\dots,s^{(t-1)}]$ which consists of a prompt and the first $t-1$ LM-generated tokens. Then it outputs a logit vector, corresponding to each token 
% $t_i$ (where $i=1,\dots,|\mathcal{V}|$) 
in $\mathcal{V}$.
These logits are then converted into a discrete probability distribution $\boldsymbol{p}^{(t)} = (p^{(t)}_1,\dots,p^{(t)}_{|\mathcal{V}|})$ over the vocabulary, via a $\mathrm{softmax}$ operator (commonly).
The next token is then sampled from $\boldsymbol{p}^{(t)}$ using either standard multinomial sampling, beam search, or other strategies. 

\subsection{LM-based Steganography}
Alice (the sender) wants to communicate a secret message $m_s \sim$ $U(\{0,1\}^l)$ with Bob (the receiver) by embedding it in a natural-language text $t_s$ (a stegotext).
% The uniform distribution is chosen for $m_s$ without loss of generality: if $m_s$ has additional structure it can be further compressed to a uniformly distributed random variable~\cite{10.1109/TIT.2004.840860}. 
Alice and Bob have agreed on an embedding function $\mathcal{S}_{\text{emb}}$ and an extracting function $\mathcal{S}_{\text{ext}}$ that perform steganography, achieved by a language model, $\mathcal{M}$. These two functions are supposed to be invertible. In other words, $\mathcal{S}_{\text{emb}}(\mathcal{M},m_s) = t_s$, $\mathcal{S}_{\text{ext}}(\mathcal{M},t_s)=m_s^{\prime}$.
% \footnote{In this work, we do not consider disambiguation methods~\cite{nozaki-murawaki-2022-addressing,yan2023A,qi2024provably} that focus on maintaining $m_s = m_s^{\prime}$, since this work is orthogonal to disambiguation.}

\subsection{Security of Steganography}
\label{sec: Security of Steganography}

\citet{10.1007/3-540-49380-8_21} first modeled steganographic security from the perspective of information theory, where given an object $\mathbf{x}$, the security of a stegosystem can be quantified by Kullback-Leibler (KL) divergence between the cover distribution (the channel distribution) \(P_c\) and the stego distribution \(P_s\),
\begin{equation}
    D_{\text{KL}}(P_c||P_s) = \sum_{\mathbf{x} \in \mathcal{C}}P_c(\mathbf{x})\log\frac{P_c(\mathbf{x})}{P_s(\mathbf{x})}
\end{equation}
which typically measures how different the two distributions are. When $D_{\text{KL}} (P_c||P_s) = 0$, the stegosystem is considered to be \textit{perfectly secure} in this perspective of information theory.

Benefiting from the explicit generative models that can predict probability distributions, the above definition of steganographic security can be modeled into another goal, that is, steganography is indistinguishable from the normal generation process, i.e., random sampling~\cite{10179287}.

In addition, from the perspective of computational security~\cite{10.1007/3-540-45708-9_6, katzenbeisser2002defining}, steganography is secure against chosen hiddentext attacks, if for all probabilistic polynomial time (PPT) adversary detection $\mathcal{A}_{\mathcal{D}}$, it holds:
\begin{equation}\label{eq: computational security}
    |\text{Pr}[\mathcal{A}_{\mathcal{D}}(x_s) = 1] - \text{Pr}[\mathcal{A}_{\mathcal{D}}(x_c) = 1]| < \text{negl}(\lambda)
\end{equation}
where $x_s$ is the stegotext, $x_c$ is the normally generated cover text, $\lambda$ is the security parameter of the shared key $K$ (usually the length of $K$), and $\text{negl}(\lambda)$ is a negligible function concerning $\lambda$.

\subsection{Statistical Imperceptibility of LM-based Steganography}
\label{sec: imperceptibility_GLS}

Following the previous formulation~\cite{dai-cai-2019-towards, shen-etal-2020-near}, statistical imperceptibility refers to the similarity between the true language model $\mathcal{M}^t$ in the monitored channel and $\mathcal{M}^s$, the language model $\mathcal{M}$ integrated with steganographic algorithms. Specifically, the total variation distance (TVD) is used to measure statistical imperceptibility. Consider the TVD between $\mathcal{M}^t$ and $\mathcal{M}^s$, i.e. $d(\mathcal{M}^t, \mathcal{M}^s)$, by triangle inequality:
\begin{equation}\label{eq: triangle}
    d(\mathcal{M}^t, \mathcal{M}^s) \leq d(\mathcal{M}^t, \mathcal{M}) + d(\mathcal{M}, \mathcal{M}^s).
\end{equation}
As $d(\mathcal{M}^t, \mathcal{M})$ is a criterion to measure the original language model, which is limited by the research on language models. Thus, $d(\mathcal{M}, \mathcal{M}^s)$ is the main focus of linguistic steganography.

According to Pinsker’s inequality~\cite{1201071} and additivity of KL divergence, $d(\mathcal{M}, \mathcal{M}^s)$ can be further decomposed in each step, that is:
\begin{equation}\label{eq: KLD}
    d(\mathcal{M}, \mathcal{M}^s) \leq \sqrt{\frac{\ln{2}}{2}\sum_{t=1}^{\infty}D_{\text{KL}}(\boldsymbol{p}^{(t)}||\boldsymbol{\hat{p}}^{(t)})}
\end{equation}
where $\boldsymbol{p}^{(t)}$ is the original probability distribution at $t^{th}$ step, and $\boldsymbol{\hat{p}}^{(t)}$ is transformed from $\boldsymbol{p}^{(t)}$ via sampling and encoding.
Hence,  linguistic steganography could aim to minimize $D_{\text{KL}}(\boldsymbol{p}^{(t)}||\boldsymbol{\hat{p}}^{(t)})$, in order to obtain relative near-imperceptibility. 
Some derivation is skipped here, as details are verified in~\cite{dai-cai-2019-towards, shen-etal-2020-near, 1201071}.

\begin{figure}[!t]
 \centering
 \includegraphics[width=\columnwidth]{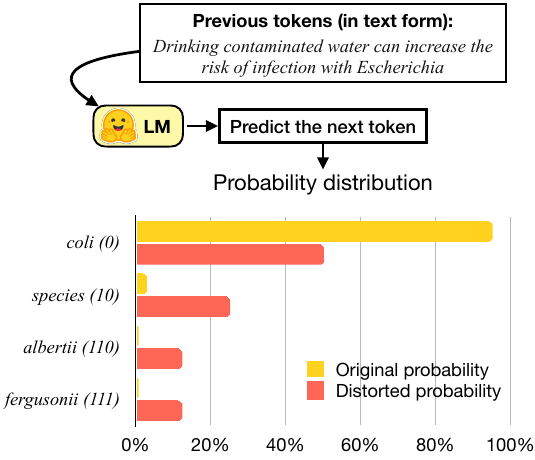} 
 \caption{An example of how steganographic encoding (using Huffman coding~\cite{8470163}) can alter the original top-4 probability distribution predicted by a language model. Specifically, a pair such as \textit{coli (0)} denotes a candidate token along with its  Huffman code, and the distorted probabilities are computed based on a random secret message.}
 \label{fig: distribution_distortion}
\end{figure}

Specifically, Figure~\ref{fig: distribution_distortion} illustrates how distribution distortion can compromise imperceptibility. In this low-entropy example, ``\textit{coli}'' has an overwhelmingly high probability of being the next token, making any alternative token extremely unlikely. If steganography introduces distribution distortion (such as through Huffman encoding), it may create detectable patterns that could be exploited by a steganalysis detector.

% In summary, $D_{\text{KL}}(\boldsymbol{p}^{(t)}||\boldsymbol{\hat{p}}^{(t)}) = 0$ (for each $t$) is a sufficient condition for achieving \textit{near-imperceptible} steganography. Besides, $D_{\text{KL}}(\boldsymbol{p}^{(t)}||\boldsymbol{\hat{p}}^{(t)}) = 0$ (for each $t$) also implies indistinguishability from random sampling, thereby satisfying the requirement for \textit{perfect security}.

\subsection{Related Work with Zero KL Divergence}
Achieving $D_{\text{KL}}(\boldsymbol{p}^{(t)}||\boldsymbol{\hat{p}}^{(t)}) = 0$ (for each $t$) is a desirable objective in steganography.
\citet{10179287} proposed Discop, a representative zero-KL provably secure steganographic method. At each generative step, the message determines which \textit{distribution copy} to sample.
However, Discop must avoid overlaps between its rotated copies to keep extraction unique: whenever overlaps occur, it embeds fewer bits.
% It further employs an iterative Huffman-tree-based method to improve capacity, achieving high entropy utilization, though the $O(|\mathcal{V}|)$ complexity of constructing a Huffman tree at each step limits efficiency when encoding large vocabularies. 
Recently, \citet{wang2025sparsampefficientprovablysecure} introduced SparSamp, which embeds messages by combining them with pseudo-random numbers to obtain message-derived random numbers for sampling.
To achieve uniquely extractable, SparSamp must sparsify its sampling grid, which leaves a small capacity gap.
More related methods and details are shown in Appendix~\ref{sec: Related Work}. 
% These methods, which process the secret message as discrete bit strings, inevitably incur a trade-off, leading to potential losses in \textit{either entropy utilization or imperceptibility}.

% \subsection{Steganography System}
% Steganography is usually illustrated by Simmons’ ``Prisoners’ Problem''~\cite{Simmons1984}: Alice and Bob (steganographers) are in jail, trying to hatch an escape plan. The only way they can communicate with each other is carefully censored by the warden Eve (steganalyzer). Once Eve detects any ``unusual'' such as illegal words, encrypted messages, or abnormal codes, she will block their plan and throw them into high-security solitary confinement. Therefore, they must find some way to embed the secret message into an ``innocent-looking'' cover-object to obtain a stego-object.

\section{Vanilla Range-Coding Steganography}
\label{sec: Vanilla Range-Coding Steganography}
In this section, we tentatively start by describing a simple \textit{vanilla} version of range-coding (RC) steganography, which directly applies RC to linguistic steganography. 

\subsection{Embedding \& Extraction}
Figure~\ref{fig: Vanilla_RC_Stega} briefly illustrates how vanilla RC steganography embeds a message into a text. In range coding, all the information can be represented in decimals and ranges (intervals). In this example, the 16-bit message is interpreted as an integer whose decimal representation is 20219, and the interval is initialized as $[0,2^{16}) = [0, 65536)$.

Algorithm~\ref{algorithm: vanilla RC embed} outlines how the sender (Alice) embeds the secret message $m_s$ into the text $t_s$ using vanilla RC steganography. 
Specifically, $m_s$ is first decimalized to $d_s$ in Line 1, and all subsequent procedures operate directly on $d_s$ rather than on a bitstream.
% Inspired by Discop~\cite{10179287}, we employ a pseudo-random number generator (PRNG) and a symmetric key $K$ to generate pseudo-random numbers for the following sampling (Line 2), ensuring reproducibility and correct extraction.
In Line 2, the initial interval is set to $[0,2^l)$, where $l$ is the length of $m_s$.
During subsequent iterative processes (Lines 3--9), the interval is progressively narrowed at each step. The iteration ends when the midpoint of the interval is exactly rounded to $d_s$ (which ensures \textbf{uniqueness}).

Algorithm~\ref{algorithm: vanilla RC extract} outlines how the receiver (Bob) extracts the secret message $m_s$ from the received text $t_s$. The initial interval is narrowed according to each token received, and $d_s$ is the rounded value of the midpoint of the final interval. Finally, the extraction result $m_s$ is binarized from $d_s$.

\begin{figure}[!t]
 \centering
 \includegraphics[width=\columnwidth]{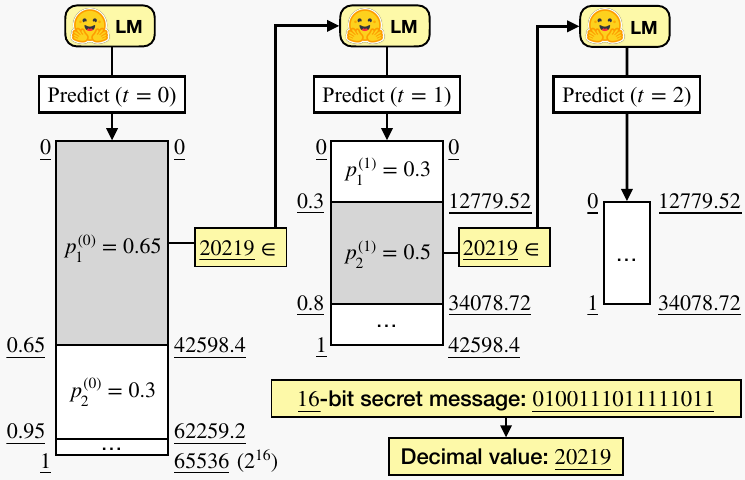} 
 \caption{An example of procedures for embedding a 16-bit secret message into a text via the vanilla RC steganography. The interval is iteratively narrowed until it can \textbf{uniquely} represent the decimal value 20219.}
 \label{fig: Vanilla_RC_Stega}
\end{figure}

\begin{algorithm}[!t]
\small
\caption{Vanilla RC steganography (embed)}\label{algorithm: vanilla RC embed} 
\textbf{Input:}\\
Context (initial previous tokens), $C$ \\
Language model, $\mathcal{M}$ \\
Message length, $l$ \\
Secret message, $m_s$  \\

{\textbf{Output:}}\\
Steganographic text, $t_s$\\ 

\begin{algorithmic}[1]
% \STATE Initialize previous tokens $$ using $s^{(-N_p)},\dots,s^{(-1)}$.
\STATE $d_s \leftarrow \mathrm{bin2dec}(m_s)$; \hfill  $\mathtt{// \ Decimalize}$
\STATE $[L,R) \leftarrow [0,2^{l})$; \hfill  $\mathtt{// \ Initialize \ interval}$
\WHILE{$\mathrm{round}(\frac{L + R}{2}) \neq d_s$}
    \STATE $\boldsymbol{p}^{(t)} \leftarrow \mathcal{M}(C)$; \hfill  $\mathtt{// \ Predict \ probs}$
    \STATE $\boldsymbol{c}^{(t)} \leftarrow 0||\boldsymbol{p}^{(t)}.\mathrm{cumsum()}$; \hfill  $\mathtt{// \ Cumulate \ probs}$
    \STATE $\boldsymbol{c}^{\prime(t)} \leftarrow L+(R-L) \times \boldsymbol{c}^{(t)}$; \hfill  $\mathtt{// \ Rescale}$
    \STATE Select $\mathrm{token}_i$ so that $d_s \in [\boldsymbol{c}^{\prime(t)}[i-1],\boldsymbol{c}^{\prime(t)}[i])$;
    \STATE $[L,R) \leftarrow [\boldsymbol{c}^{\prime(t)}[i-1],\boldsymbol{c}^{\prime(t)}[i])$;
    \STATE $C \leftarrow C||\mathrm{token}_i$;
\ENDWHILE

\STATE Detokenize $C$ into $t_s$;
\RETURN $t_s$

\end{algorithmic}

\end{algorithm}

\begin{algorithm}[!t]
\small
\caption{Vanilla RC steganography (extract)}\label{algorithm: vanilla RC extract} 
\textbf{Input:}\\
Context (initial previous tokens), $C$ \\
Language model, $\mathcal{M}$ \\
Message length, $l$ \\
Steganographic text, $t_s$\\ 

{\textbf{Output:}}\\
Secret message, $m_s$  \\

\begin{algorithmic}[1]
% \STATE Initialize previous tokens $$ using $s^{(-N_p)},\dots,s^{(-1)}$.
\STATE Tokenize $t_s$ into $S$;
\STATE $[L,R) \leftarrow [0,2^{l})$; \hfill  $\mathtt{// \ Initialize \ interval}$
% \STATE $t \leftarrow 0$;
\FOR{$t = 0, 1, \ldots, |S| - |C| - 1$}
    \STATE $\boldsymbol{p}^{(t)} \leftarrow \mathcal{M}\left(S\left[:|C|+t\right]\right)$; \hfill  $\mathtt{// \ Predict \ probs}$
    \STATE $\boldsymbol{c}^{(t)} \leftarrow 0|| \boldsymbol{p}^{(t)}.\mathrm{cumsum()}$; \hfill  $\mathtt{// \ Cumulate \ probs}$
    \STATE $\boldsymbol{c}^{\prime(t)} \leftarrow L+(R-L) \times \boldsymbol{c}^{(t)}$; \hfill  $\mathtt{// \ Rescale}$
    \STATE Select $\mathrm{token}_i$ so that $\mathrm{token}_i = S\left[|C|+t + 1\right]$;
    \STATE $[L,R) \leftarrow [\boldsymbol{c}^{\prime(t)}[i-1],\boldsymbol{c}^{\prime(t)}[i])$;
    % \STATE $t \leftarrow t + 1$;
\ENDFOR

\STATE $d_s \leftarrow \mathrm{round(\frac{L+R}{2})}$;
\STATE $m_s \leftarrow \mathrm{dec2bin}(d_s).\mathrm{zfill}(l)$; \hfill  $\mathtt{// \  Binarize \ \& \ Fill \ 0}$
\RETURN $m_s$

\end{algorithmic}

\end{algorithm}

\subsection{Security Issues}
\label{sec: Security Issues}
For vanilla RC steganography, there can be two security issues:

\textit{1) Distortion on probability distribution}. Taking the first generative step in Figure~\ref{fig: Vanilla_RC_Stega} as an example, the (softmax) probability of $\mathrm{token}_1 \ (t=0)$, $p_1^{(0)}$, is $0.65$, and its interval is $[0,42598.4)$. Considering a random $16$-bit secret message $m_s \sim U(\{0,1\}^{16})$, so $d_s \sim U(\{0,1,\ldots,2^{16}-1\})$ ($d_s$ is a \textbf{discrete}  uniform random variable). Then, the steganographic sampled probability for $\mathrm{token}_1 \ (t=0)$ is
$P\left(d_s \in [0,42598.4)\right) = \frac{42599}{65536} \neq 0.65 = p_1^{(0)}$.
Thus, distortion on probability distribution occurs and zero KL divergence or perfect security cannot hold. 
Even though it could be mitigated when lengthening $m_s$ ($l$ can be set greater than the tensor precision). However, similar to AC-based steganography, iterative interval narrowing still causes noticeable distortion in small intervals.

\textit{2) Randomness reuse}. According to~\citet{10.1145/3460120.3484550}, reusing randomness in multiple sampling events could expose features and bias to detectors.
Therefore, only using the bits of the message as the randomness or encrypting the message with a pseudo-random cipher, as in a public-key solution, is insecure because multiple samplings will be forced to reuse randomness. 

\begin{algorithm}[!t]
\small
\caption{Rotation RC steganography (embed)}\label{algorithm: Rotation RC embed} 
\textbf{Input:}\\
Context (initial previous tokens), $C$ \\
Pseudo-random number generator, $\mathrm{PRNG}$ \\
Language model, $\mathcal{M}$ \\
Symmetric key (seed), $K$ \\
Message length, $l$ \\
Secret message, $m_s$  \\

{\textbf{Output:}}\\
Steganographic text, $t_s$\\ 

\begin{algorithmic}[1]
% \STATE Initialize previous tokens $$ using $s^{(-N_p)},\dots,s^{(-1)}$.
\STATE $d_s^{(-1)} \leftarrow \mathrm{bin2dec}(m_s)$; \hfill  $\mathtt{// \ Decimalize}$
\STATE $\mathrm{PRNG.set\_seed}(K)$;
\STATE $[L^{(-1)},R^{(-1)}) \leftarrow [0,2^{l})$; \hfill  $\mathtt{// \ Initialize \ interval}$
\FOR{$t = 0,1,\ldots$}
    \STATE $\boldsymbol{p}^{(t)} \leftarrow \mathcal{M}(C)$; \hfill  $\mathtt{// \ Predict \ probs}$
    \STATE $\boldsymbol{c}^{(t)} \leftarrow 0 || \boldsymbol{p}^{(t)}.\mathrm{cumsum()}$; \hfill  $\mathtt{// \ Cumulate \ probs}$
    \STATE $\Delta^{(t-1)} = R^{(t-1)} - L^{(t-1)}$;
    \STATE $\boldsymbol{c}^{\prime(t)} \leftarrow L^{(t-1)}+ \Delta^{(t-1)} \times \boldsymbol{c}^{(t)}$; \hfill  $\mathtt{// \ Rescale}$
    \STATE $o^{(t)} \leftarrow U(0,1).\mathrm{sample}(\mathrm{PRNG^{(t)}})$;

    \STATE $d_s^{(t)} \leftarrow L^{(t-1)}+ (d_s^{(t-1)} - L^{(t-1)} + o^{(t)} \times \Delta^{(t-1)}) \bmod \Delta^{(t-1)} $; \hfill  $\mathtt{// \ Rotate}$
    \STATE Select $\mathrm{token}_i$ so that $d_s^{(t)} \in [\boldsymbol{c}^{\prime(t)}[i-1],\boldsymbol{c}^{\prime(t)}[i])$;
    \STATE $[L^{(t)},R^{(t)}) \leftarrow [\boldsymbol{c}^{\prime(t)}[i-1],\boldsymbol{c}^{\prime(t)}[i])$;
    \STATE $C \leftarrow C||\mathrm{token}_i$;
    \IF{$\frac{L^{(t)} +R^{(t)}}{2} - d_s^{(t)} \in (-0.5,0.5] $}
        \STATE \textbf{break}
    \ENDIF
\ENDFOR

\STATE Detokenize $C$ into $t_s$;
\RETURN $t_s$

\end{algorithmic}

\end{algorithm}

\begin{algorithm}[!t]
\small
\caption{Rotation RC steganography (extract)}\label{algorithm: Rotation RC extract} 
\textbf{Input:}\\
Context (initial previous tokens), $C$ \\
Pseudo-random number generator, $\mathrm{PRNG}$ \\
Language model, $\mathcal{M}$ \\
Symmetric key (seed), $K$ \\
Message length, $l$ \\
Steganographic text, $t_s$\\ 

{\textbf{Output:}}\\
Secret message, $m_s$  \\

\begin{algorithmic}[1]
% \STATE Initialize previous tokens $$ using $s^{(-N_p)},\dots,s^{(-1)}$.
\STATE Tokenize $t_s$ into $S$;
\STATE $\mathrm{PRNG.set\_seed}(K)$;
\STATE $[L^{(-1)},R^{(-1)}) \leftarrow [0,2^{l})$; \hfill  $\mathtt{// \ Initialize \ interval}$
\STATE $t_{\text{end}} \leftarrow |S| - |C| - 1$;
\FOR{$t = 0, 1, \ldots, t_{\text{end}}$}
    \STATE $\boldsymbol{p}^{(t)} \leftarrow \mathcal{M}\left(S\left[:|C|+t\right]\right)$; \hfill  $\mathtt{// \ Predict \ probs}$
    \STATE $\boldsymbol{c}^{(t)} \leftarrow 0 || \boldsymbol{p}^{(t)}.\mathrm{cumsum()}$; \hfill  $\mathtt{// \ Cumulate \ probs}$
    \STATE $\Delta^{(t-1)} \leftarrow R^{(t-1)} - L^{(t-1)}$;
    \STATE $\boldsymbol{c}^{\prime(t)} \leftarrow L^{(t-1)}+ \Delta^{(t-1)} \times \boldsymbol{c}^{(t)}$; \hfill  $\mathtt{// \ Rescale}$
    \STATE Select $\mathrm{token}_i$ so that $\mathrm{token}_i = S\left[|C|+t + 1\right]$;
    \STATE $[L^{(t)},R^{(t)}) \leftarrow [\boldsymbol{c}^{\prime(t)}[i-1],\boldsymbol{c}^{\prime(t)}[i])$;
    
\ENDFOR

\STATE $\mathrm{mid}^{(t_{\text{end}})} \leftarrow (L^{(t_{\text{end}})} + R^{(t_{\text{end}})}) / 2$; 
\FOR{$t =t_{\text{end}}, \ldots, 1,0$}
    \STATE $o^{(t)} \leftarrow U(0,1).\mathrm{sample}(\mathrm{PRNG^{(t)}})$;
    \STATE $\mathrm{mid}^{(t-1)} \leftarrow L^{(t-1)} + (\mathrm{mid}^{(t)} - L^{(t-1)} - o^{(t)} \times \Delta^{(t-1)}) \bmod \Delta^{(t-1)}$; \hfill  $\mathtt{// \ Rotate \ reversely}$
\ENDFOR
\STATE $d_s^{(-1)} \leftarrow \mathrm{round\_half\_down}(\mathrm{mid}^{(-1)})$;

\STATE $m_s \leftarrow \mathrm{dec2bin}(d_s^{(-1)}).\mathrm{zfill}(l)$; \hfill  $\mathtt{// Binarize \ \& \ Fill \ 0}$
\RETURN $m_s$

\end{algorithmic}

\end{algorithm}

\section{Rotation Range-Coding Steganography}
\label{sec: Rotation Range-Coding Steganography}
Considering the security issues discussed above, we propose a rotation range-coding (RRC) steganographic method.
Instead of directly using the constant $d_s$, our proposed rotation mechanism updates $d_s^{(t-1)}$ to $d_s^{(t)}$ at each time step $t$ (initial $d_s = d_s^{(-1)}$), with the following objectives:
\begin{itemize}
    \item To transform the discrete uniform random variable $d_s$ to a \textbf{continuous} uniform random variable $d_s^{(t)}$ at each $t$, thereby preserving the original probability distribution and ensuring zero KL divergence.
    \item To introduce \textit{fresh} randomness at each $t$, thereby preventing the reuse of randomness. 
\end{itemize}

% In this section, we introduce the procedures of our RRC steganography, and explain how it overcomes security issues in vanilla RC steganography.
% Additionally, we explain embedding efficiency of our method, that is, how its embedding capacity can achieve 100\% entropy utilization.

\subsection{Embedding of RRC Steganography}

Algorithm~\ref{algorithm: Rotation RC embed} outlines the embedding procedures of RRC steganography (Algorithm~\ref{algorithm: Rotation RC embed} shares several steps with Algorithm~\ref{algorithm: vanilla RC embed}).
% Specifically, $m_s$ is first decimalized to $d_s$ in Line 1, and all subsequent procedures operate directly on $d_s$ rather than on a bitstream.
Inspired by other provably secure methods, we employ a pseudo-random number generator (PRNG) and a symmetric key $K$ to generate pseudo-random numbers for the following sampling (Line 2), ensuring reproducibility and correct extraction.
% In Line 3, the initial interval is set to $[0,2^l)$, where $l$ is the length of $m_s$.
% During the subsequent iterative processes (Lines 5--13), the interval is progressively narrowed at each step. The terminal condition is met when the average of upper and lower boundaries of the interval rounds exactly to $d_s$.\footnote{$\mathrm{round(.)}$ refers to $\mathrm{round\_half\_down}(.)$ in this paper, and for any $m \in \mathbb{Z}$ $\mathrm{round\_half\_down}( m + 0.5) = m$.} 
Note that the pseudo-random numbers generated by PRNG are used to control the sampling of the offset $o \sim U(0,1)$ at each $t$ (Line 9), and then $o$ is used to rotate $d_s^{(t-1)}$ to $d_s^{(t)}$ (Line 10). Besides, the termination condition in RRC steganography is $\frac{L^{(t)} +R^{(t)}}{2} - d_s^{(t)} \in (-0.5,0.5] $ (Lines 14--15), which is tailored for \textit{\textbf{unique extraction} and avoiding generating unnecessary tokens}.

\subsection{Extraction of RRC Steganography}

Algorithm~\ref{algorithm: Rotation RC extract} outlines how the receiver (Bob) extracts the secret message $m_s$ from the steganographic text $t_s$. As Alice and Bob have agreed on $\mathrm{PRNG}$ and symmetric key $K$, Bob can synchronize each $t$-time rotation with Alice and reproduce each range of $d_s^{(t)}$ according to each interval $[L^{(t)}, R^{(t)})$ at each $t$. Based on the termination condition of the embedding algorithm, the end time $t_{\text{end}} = |S| - |C| - 1$, and $\mathrm{mid}^{(t_{\text{end}})} = (L^{(t_{\text{end}})} + R^{(t_{\text{end}})}) / 2$, there is:
\[\mathrm{mid}^{(t_{\text{end}})} - d_s^{(t_{\text{end}})} \in (-0.5,0.5] \]
\[d_s^{(t_{\text{end}})} \in [\mathrm{mid}^{(t_{\text{end}})} - 0.5, \mathrm{mid}^{(t_{\text{end}})} + 0.5).\]

Considering linear transformation and the rotation in embedding, for each $t$ there is (in Line 15): 
% $\mathrm{mid}^{(t-1)} =  L^{(t-1)} +  (\mathrm{mid}^{(t)} - L^{(t-1)} - o^{(t)} \times  \Delta^{(t-1)})  \bmod \Delta^{(t-1)} $
\begin{align*}
& \mathrm{mid}^{(t-1)} =  L^{(t-1)} +  (\mathrm{mid}^{(t)} - L^{(t-1)} - o^{(t)} \times  \\ &  \Delta^{(t-1)})  \bmod \Delta^{(t-1)} 
\end{align*}
\[d_s^{(t)} \in [\mathrm{mid}^{(t)} - 0.5, \mathrm{mid}^{(t)} + 0.5)\]
where $\Delta^{(t-1)} = R^{(t-1)} - L^{(t-1)}$.
After iteration, as $d_s^{(-1)} \in \mathbb{Z}$, there is (Line 16):
\[d_s^{(-1)} \in [\mathrm{mid}^{(-1)} - 0.5, \mathrm{mid}^{(-1)} + 0.5) \]
\[d_s^{(-1)} = \mathrm{round\_half\_down}(\mathrm{mid}^{(-1)})\]
where $\mathrm{round\_half\_down}$ means that when a number is exactly halfway between two possible rounded values (e.g., $2.5$), $\mathrm{round\_half\_down}$ rounds toward the smaller rounded values (e.g., $2$).

Therefore, in RRC steganography, $d_s^{(-1)}$ can be computed \textbf{uniquely} by Bob, and then $m_s$ is binarized from $d_s^{(-1)}$ (Line 17).

\section{Analysis of RRC Steganography}
\label{sec: Analysis of RRC Steganography}
% In this section, we prove the zero KL divergence and computational security of RRC steganography and analyze its capacity and complexity.

\subsection{Proof of Zero KL Divergence}
\label{sec: Proof of Zero KL Divergence}
Considering rotation, we first introduce Proposition~\ref{proposition: continuous uniform} (the rigorous proof is shown in Appendix~\ref{sec: Proof of Proposition 1}), and then explain how the original probability distribution is preserved (Proposition~\ref{proposition: 0 KLD}).
\begin{proposition}\label{proposition: continuous uniform}
$d_s^{(t)} \sim U(L^{(t-1)}, R^{(t-1)})$.
\end{proposition}
\begin{proposition}\label{proposition: 0 KLD}
In Line 11 (Algorithm~\ref{algorithm: Rotation RC embed}), the selected probability of each $\mathrm{token}_i$ equals $p^{(t)}_i$.
\end{proposition}
\begin{proof}
    Considering the interval construction from $\boldsymbol{p}^{(t)}$ to $\boldsymbol{c}^{\prime(t)}$ (Lines 5--8 in Algorithm~\ref{algorithm: Rotation RC embed}), $\boldsymbol{p}^{(t)} = (p^{(t)}_1, \ldots,p^{(t)}_{|\mathcal{V}|})$, $\boldsymbol{c}^{(t)} = (c^{(t)}_o,c^{(t)}_1,\ldots,c^{(t)}_{|\mathcal{V}|}) = (0,\sum_{k=1}^1 p_k, \ldots,\sum_{k=1}^{|\mathcal{V}|}p_k)$, and $\boldsymbol{c}^{\prime(t)} = L^{(t-1)} + \Delta^{(t-1)} \times \boldsymbol{c}^{(t)} = (L^{(t-1)}, L^{(t-1)} + \Delta^{(t-1)} \times \sum_{k=1}^1 p_k, \ldots, R^{(t-1)})$.
    According to Proposition~\ref{proposition: continuous uniform}, $d_s^{(t)} \sim U(L^{(t-1)}, R^{(t-1)})$, so the probability density function is:
    \[
        f_{d_s^{(t)}}(x) =
        \begin{cases}
        \frac{1}{\Delta^{(t-1)}}, & \text{if } x\in [L^{(t-1)},R^{(t-1)}),\\
        0,  & \text{otherwise}
        \end{cases}
    \]
    \begin{align*}
       & P\left(d_s^{(t)} \in [\boldsymbol{c}^{\prime(t)}[i-1],\boldsymbol{c}^{\prime(t)}[i])\right) \\
       & = \int_{\boldsymbol{c}^{\prime(t)}[i-1]}^{\boldsymbol{c}^{\prime(t)}[i]} f_{d_s^{(t)}}(x)\,dx  = \frac{\boldsymbol{c}^{\prime(t)}[i] - \boldsymbol{c}^{\prime(t)}[i-1]}{\Delta^{(t-1)}} \\
       & = \frac{(\Delta^{(t-1)} \times \sum_{k=1}^i p_k) - ( \Delta^{(t-1)} \times \sum_{k=1}^{i-1} p_k)}{\Delta^{(t-1)}} \\
       & =  \sum_{k=1}^i p_k - \sum_{k=1}^{i-1} p_k = p_i.
    \end{align*}
    
    Therefore, the probability of selecting $\text{token}_i$ is exactly $p_i^{(t)}$, the original LM distribution.
\end{proof}
As RRC steganography does not change the original predicted probability by the LM for each token, the KL divergence between the original and the steganographic probability distribution is zero.
RRC steganography possesses perfect statistical imperceptibility (according to Section~\ref{sec: imperceptibility_GLS}).

\subsection{Computational Security}
From the perspective of computational security (Section~\ref{sec: Security of Steganography} and Equation~\ref{eq: computational security}), we prove our RRC steganography is a secure method, for which details are shown in Appendix~\ref{sec: Proof of Computational Security of RRC Steganography}.

\subsection{Embedding Capacity}
\label{sec: Embedding Capacity}
% Embedding Capacity of LM-based steganography is represented by how many secret bits can be averagely embedded into each token, that is, \textit{bits per token} ($\mathrm{BPT}$). Given the length of the secret message $l$ and the generated token number $T$, $\mathrm{BPT} = {l}/{T}$.

% According to algorithms of our RC-based steganography, let $[0,2^l)$ be the initial interval containing a target integer $d_s$. At each step $t$,  a probability distribution $\{p^{(t)}_i\}$ with entropy $H^{(t)} = -\sum_i p^{(t)}_i \log_2 p^{(t)}_i$ is randomly rotated and rescaled to partition the current interval. The process selects the subinterval $[L^{(t)}, R^{(t)})$ containing $d_s$ and repeats until $\mathrm{round}(\frac{L^{(t)} + R^{(t)}}{2}) = d_s$. Note that $\mathrm{round(.)}$ refers to $\mathrm{round\_half\_down}(.)$.
First, we consider when the embedding ends according to the proposed termination condition, there is a proposition (its proof is shown in Appendix~\ref{sec: Proof of Proposition 3}):
\begin{proposition}\label{proposition: termination}
$\Delta^{(t)} \leq 1$ is a sufficient condition for the embedding termination $\frac{L^{(t)} +R^{(t)}}{2} - d_s^{(t)} \in (-0.5,0.5] $ (Lines 14--15 in Algorithm~\ref{algorithm: Rotation RC embed}).
\end{proposition}

Then, given the initial interval $[L^{(-1)}, R^{(-1)}) = [0,2^n)$ and the interval length $\Delta^{(-1)} = 2^n$, the interval length at $t$ is:
\begin{equation}\label{eq: narrow}
    \Delta^{(t)} = 2^n \cdot \prod_{i=0}^t p^{(i)}_{\mathrm{output}}
\end{equation}
and there is: 
\begin{equation}
\frac{\Delta^{(t)}}{\Delta^{(t-1)}} = p^{(t)}_{\mathrm{output}}
\end{equation}
where $p^{(i)}_{\mathrm{output}}$ is the probability of the output token ($i = 0,1,\ldots,t$).

According to Proposition~\ref{proposition: termination}, when $\Delta^{(t)} \leq 1$, the embedding iteration ends. Considering information theory and Equation~\ref{eq: narrow}, the interval shrinkage rate is determined by the entropy of the probability distribution. The average amount of information per iteration is $H^{(t)}$, and the total amount of information is required to cover $n$ bits of the initial interval. Therefore, the number of loops is satisfied:
$\sum_{i=0}^t H^{(t)} \geq n$ 
where $H^{(t)} = -\sum^{|\mathcal{V}|}_{i=1} p^{(t)}_i \log_2 p^{(t)}_i $, and let the average entropy is $H_{\mathrm{avg}}$, so that the loop number (which is exactly the number of the generated tokens) is: $N_{\mathrm{token}} \approx \frac{n}{H_{\mathrm{avg}}}$.
Thus, the embedding capacity (bits per token) can be represented as: $\frac{n}{N_{\mathrm{token}}} \approx H_{\mathrm{avg}}$.
RRC steganography can achieve approximate \textbf{100\% utilization of entropy}.

\subsection{Complexity}
% To make steganographic behavior as close to random sampling as possible, our method only considers using the entire vocabulary $\mathcal{V}$ as the candidate pool at each generative step. 
RC-based steganography (including vanilla RC steganography and RRC steganography) requires updating probability intervals after each step, resulting in a time complexity of $O(|\mathcal{V}|)$ for each generative step. The complexity can be computed following the approaches used in AC steganography~\cite{ziegler-etal-2019-neural}.

\begin{table*}[!t]
% \small
\renewcommand{\arraystretch}{1.0}
\centering
\scalebox{0.9}{
\begin{tabular}{l|ccccc}
\toprule[1.0pt]
Method          & \begin{tabular}[c]{@{}c@{}} Avg / Max KLD \\ (bits/token) \end{tabular} $\downarrow$ & \begin{tabular}[c]{@{}c@{}}Capacity  \\  (bits/token)\end{tabular} $\uparrow$ & \begin{tabular}[c]{@{}c@{}}Entropy \\  (bits/token)\end{tabular} & \begin{tabular}[c]{@{}c@{}}Utilization \\ (\%)  \end{tabular}  $\uparrow$ & \begin{tabular}[c]{@{}c@{}}Speed  \\ (bits/s) \end{tabular} $\uparrow$ \\  \hline \rowcolor{gray!20}
Multinomial sampling &   0 / 0      &     N/A    &  5.86  &     N/A   &     N/A    \\ 
\hline
AC              &  1.95E-03 / 3.01E-02   &   \underline{5.86}       &   5.87  &  \underline{99.83}  &   1025.36   \\
ADG             &  1.60E-04 / 1.57E-03    &    4.81     &   5.89 &  81.60  & 36.45 \\
Meteor w/o sort  &    4.22E-02 / 1.16E-01     &   4.17    &  5.79 &  71.96 &  950.27                                                        \\
Meteor w/ sort &    4.11E-02 / 1.16E-01  & 4.77  & 5.81   &  82.08   &  25.25   \\
\hline
iMEC    &   \textbf{0 / 0} &      4.16    &     5.83     &   71.44         &  27.30   \\
Discop w/o sort  &   \textbf{0 / 0}   &    2.31  &  5.90                    &      39.31    &   218.34  \\
Discop w/ sort &    \textbf{0 / 0}     &    5.58     &      5.86   &  95.17    &  44.30                                                        \\
SparSamp    &   \textbf{0 / 0}      &   5.74      &    5.93     &       96.76    &  \underline{1267.82}     \\
\hline \rowcolor{orange!20} 
RRC steganography (ours)    &   \textbf{0 / 0}  & \textbf{5.93}  &    5.93   &   \textbf{99.98} &   \textbf{1554.66}      \\
% \textbf{RC} (Tensor)       &   \textbf{0 / 0}      &  5.58     &    5.91   &  94.40                                                             &    \textbf{1519.64}     \\
\bottomrule[1.0pt]
\end{tabular}}
 \caption{Quantitative comparison with previous steganographic methods on GPT-2.}
 \label{table: GPT-2_results}
\end{table*}

\section{Experiments}
\label{sec: Experiments}
To validate the security and efficiency of our RRC steganography, we evaluated it compared to a series of methods toward provable security in this era, including arithmetic coding (AC)~\cite{ziegler-etal-2019-neural}, ADG~\cite{zhang-etal-2021-provably}, Meteor~\cite{10.1145/3460120.3484550}, iMEC~\cite{witt2023perfectly}, Discop~\cite{10179287}, and SparSamp~\cite{wang2025sparsampefficientprovablysecure}.

\subsection{Setup}
\label{sec: Setup}

To validate the generalizability of our steganographic method, we implemented it using three language models of various scales: GPT-2~\cite{radford2019language},\footnote{\href{https://huggingface.co/openai-community/gpt2}{https://huggingface.co/openai-community/gpt2}} OPT-1.3b~\cite{zhang2022opt},\footnote{\href{https://huggingface.co/facebook/opt-1.3b}{https://huggingface.co/facebook/opt-1.3b}} and Llama-2-7b~\cite{touvron2023llama}.\footnote{\href{https://huggingface.co/meta-llama/Llama-2-7b-hf}{https://huggingface.co/meta-llama/Llama-2-7b-hf}}

For each language model and steganographic method, 1,000 samples were generated using 1,000 different initial contexts. These contexts consist of the first 10 words from sequences randomly selected from the C4 dataset.\footnote{\href{https://huggingface.co/datasets/allenai/c4}{https://huggingface.co/datasets/allenai/c4}}

All the experiments were conducted with top-$p$ ($p=1.0$) sampling, i.e., encoding the entire vocabulary $\mathcal{V}$, and $1.0$ temperature. Experiments were implemented in Python 3.12.7 with Torch 2.5.0, and accelerated by using RTX 6000 Ada Generation GPUs. Besides, considering the precision limitation of the tensor, we imported Python's \texttt{decimal} module for computing with sufficient precision. Otherwise without an enough precision, errors or incorrect extractions could occur.

\subsection{Metrics}

\textbf{Avg (Max) KLD}, \textit{a security metric}, refers to the average (or maximum) value of the KL divergence $D_{\text{KL}}(\boldsymbol{p}^{(t)}||\boldsymbol{\hat{p}}^{(t)})$ in all steps, which indicates the average (or maximum) degree to the original distribution by steganography~\cite{10179287}. Specifically, in $D_{\text{KL}}(\boldsymbol{p}^{(t)}||\boldsymbol{\hat{p}}^{(t)})$, $\boldsymbol{p}^{(t)}$ is the probability distribution of the original candidate pool, and $\boldsymbol{\hat{p}}^{(t)}$ is the probability distribution of the modified (steganographic) candidate pool at each $t$.

\textbf{Embedding capacity} refers to the average number of bits embedded per generated token.

\textbf{Entropy utilization (embedding efficiency)} refers to the ratio of embedding capacity to the average entropy over all steps.

\textbf{Embedding speed} refers to the average seconds required to embed a single secret bit.

\begin{table*}[!t]
% \small
\renewcommand{\arraystretch}{1.0}
\centering
\scalebox{0.85}{
\begin{tabular}{l|r|r|r|r|r|r|r|r|r}
\toprule[1.0pt]
Message length (bits)            & 32 & 64 & 128 & 256 & 512 & 1024 & 2048 & 4096 & 8192 \\
\midrule[1.0pt]
Utilization (\%) $\uparrow$    &  99.86  & 99.19   &  99.98   & 99.89    & 99.77     &   99.93   &  \underline{100.35}    &  \textbf{100.43}    &  99.99    \\
Speed (bits/s)  $\uparrow$       &  1390.23  &  1496.88  & \underline{1554.66}    & \textbf{1572.40}    & 1475.68    &  1511.83    &   1191.88   &   880.82   &  604.76   \\
Running time (s) &  0.023  &  0.043  & 0.082    & 0.163    & 0.374    &  0.677    &   1.718  &   4.650   &  13.546   \\
\bottomrule[1.0pt]
\end{tabular}}
 \caption{Average results on utilization, speed and running time of RRC steganography across various message lengths $l$ on GPT-2.}
 \label{table: scalability_results}
\end{table*}

\subsection{Main Results}
Table~\ref{table: GPT-2_results} presents the average results across various metrics for GPT-2. Besides, the results for OPT-1.3b and Llama-2-7b are shown in Tables ~\ref{table: OPT-1.3b_results} and~\ref{table: Llama-2-7b_results} in Appendix~\ref{sec: Supplementary Main Results}.
Both Meteor and Discop were evaluated in two configurations: sorted and unsorted.
For each metric, the best-performing result is highlighted in \textbf{bold}, while the second-best is indicated with \underline{underline}.
The embedded secret message was a randomly generated 128-bit sequence.
In addition, multinomial sampling generation (random sampling) was also carried out for comparison.
The key findings from these experiments are as follows:

\textit{1)} As analyzed in~\citet{wang2025sparsampefficientprovablysecure}, iMEC, Discop and SparSamp are probability-preserving. The zero KL divergence of RRC steganography is proved in Section~\ref{sec: Proof of Zero KL Divergence}, thus these methods and our RRC steganography can achieve 0 KL divergence.
% It especially makes sense to compare our method with other methods with 0 KL divergence.

\textit{2)} Our RRC steganography empirically achieves around 100\% entropy utilization, which complies with the theoretical analysis in Section~\ref{sec: Embedding Capacity}, which denotes the \textbf{100\% embedding efficiency}. Besides, the entropy utilization of our method is steadily superior to other baseline methods when implemented in different language models.

\textit{3)} Our RRC steganography achieves a highly competitive embedding speed, which is the \textbf{fastest} in GPT-2  (up to 1554.66 bits/s). However, in the other two language models, our method obtains the second-fastest speeds.
% , which are inferior to SparSamp, because SparSamp is especially characterized by its great speed and $O(1)$ complexity.

Additionally, Appendix~\ref{sec: Ablation} presents ablation studies comparing vanilla RC (Section~\ref{sec: Vanilla Range-Coding Steganography}) and provable secure RRC steganography (Section~\ref{sec: Rotation Range-Coding Steganography}).

\subsection{Scalability of RRC Steganography}
Steganography based on range coding has a distinct characteristic, that is, it embeds the entire secret message using decimal values, rather than embedding it bit by bit. In other words, the minimum unit of embedding is the complete $l$-bit message itself. If the embedding process is not completed, the message is considered not embedded at all. Therefore, scalability should be considered, as it reflects how well RRC steganography can support the secret message with various lengths.

% \begin{table}[!t]
% % \small
% \renewcommand{\arraystretch}{1.0}
% \centering
% \scalebox{0.85}{
% \begin{tabular}{l|rrr}
% \toprule[1pt]
%    Steganalysis model     & GPT-2 & OPT-1.3b & Llama-2-7b \\
% \midrule[1pt]
% bert-base-uncased &   49.3\%    &   48.1\%       &      49.7\%      \\
% roberta-base      &    51.1\%   &    50.1\%        &      51.3\%     \\ 
% roberta-large     &    48.9\%   &    52.6\%      &      47.7\%      \\
% \bottomrule[1pt]
% \end{tabular}}
%  \caption{Steganalysis accuracies against RRC steganography under cases where models for steganography vary and models for steganalysis vary.}
%  \label{table: steganalysis_results}
% \end{table}

Table~\ref{table: scalability_results} lists the average utilization, speed, and running time when RRC steganography embeds the secret message with various lengths on GPT-2. Tables~\ref{table: scalability_results_OPT-1.3b} and~\ref{table: scalability_results_Llama-2-7b} (in Appendix~\ref{sec: Supplementary Results of Scalability}) show results conducted in OPT-1.3b and Llama-2-7b.
The number of generated texts for each message length is 1000. From this table, we can find that:

\textit{1)} RRC steganography maintains steady entropy utilization around 100\%.\footnote{The stability outperforms SparSamp whose utilization is sensitive to message length and the maximum length is only 1023 (according to the data disclosed in~\citet{wang2025sparsampefficientprovablysecure}).} 

\textit{2)} When the message length varies from 64 to 1024 bits, the embedding speed is steadily around 1500 bits per second.

\textit{3)} Our method supports messages with significantly higher bit lengths (with 8192 bits not representing an upper limit), enabled by the scalable precision of Python’s \texttt{decimal} module.

\subsection{Anti-steganalysis Capacity}
In this section, we evaluated the ability of our method to evade Eve's detection using steganalysis techniques, specifically through a fine-tuned discriminator.
The discriminators used for detection were fine-tuned versions of the pretrained BERT~\cite{devlin-etal-2019-bert} and RoBERTa~\cite{DBLP:journals/corr/abs-1911-02116} models, respectively. 
Table~\ref{table: Comparison of Steganalysis} in Appendix~\ref{sec: Supplementary Results of Comparison} presents the steganalysis accuracies for steganographic texts generated by three different language models. Accuracies around 50\% indicate that the \textit{steganalysis methods do not perform better than random guessing} in detecting texts.
% produced by RRC steganography, thereby providing empirical evidence of its security.

\begin{table*}[!t]
\centering
\begin{tabular}{l|cccc}
\toprule
Model & Accuracy & Precision & Recall & F1 \\
\midrule
GPT-2       & 47.8\%  & 47.8\% & 47.2\%  & 47.5\%  \\
OPT-1.3B    & 46.6\% & 46.2\% & 41.2\% & 43.6\% \\
Llama-2-7B  & 50.6\% & 50.6\% & 52.4\% & 51.5\% \\
\bottomrule
\end{tabular}
\caption{Human evaluation results across different models.}
\label{tab:human_eval}
\end{table*}

\subsection{Human Evaluation}
We randomly mixed the steganographic texts and cover texts (250 samples per category) and asked three human evaluators to judge whether each text was machine-generated with embedded hidden information. As shown in Table~\ref{tab:human_eval}, the consistently low detection scores indicate that human evaluators struggle to distinguish steganographic texts from cover texts, providing strong evidence for the perceptual imperceptibility of RRC steganography. Notably, statistical classifiers already outperform humans at this discrimination task, as they can aggregate subtle statistical cues that are imperceptible to human readers.

\section{Conclusion}
In this paper, we explore the use of a relatively simple and classical approach, range coding (RC), to achieve provably secure steganography as well as high embedding efficiency and speed. However, two key security challenges arise: (1) distortion of the probability distribution and (2) reuse of randomness.
To address these issues, we propose rotation range-coding (RRC) steganography, and provide theoretical explanations and proofs for it.
RRC empirically outperforms the baseline methods in both embedding efficiency and capacity, while also achieving competitive embedding speed.
Moreover, RRC steganography is training-free, model-agnostic, and straightforward to implement, making it a practical foundation for future extensions to multi-modal steganography and for deployment in real-world, privacy-preserving communication.

\section*{Limitations}
In the symmetric steganographic system based on RRC steganography, Alice and Bob must agree on the secret message length $l$ before the steganographic communication, which is used to initialize the interval $[0,2^l)$ for both sides and fill ``0'' in extraction (Line 17 in Algorithm~\ref{algorithm: Rotation RC extract}).

For the analysis of embedding capacity or embedding efficiency (Section~\ref{sec: Embedding Capacity}) of RRC steganography, we only explain an approximate 100\% entropy utilization for it without rigorous theoretical proofs, but experiments can empirically prove that its utilization is approximate 100\%.

\section*{Ethical Considerations}
While steganography has legitimate applications such as embedding copyright information, it can also be misused for disinformation or to evade censorship. This dual-use nature underscores the need for effective monitoring and regulation.
We emphasize that future research should develop effective detection and monitoring systems in parallel, ensuring responsible use of steganography.

\section*{Acknowledgments}
We express our gratitude to the anonymous reviewers for their valuable and insightful comments.
This work was supported by JST SPRING, Grant Number JPMJSP2110.

\bibliography{custom}

% \newpage

\appendix
\section{Related Work}
\label{sec: Related Work}
In this section, we introduce the existing attempts to provably secure steganography, and analyze their characteristics or limitations.
\subsection{Arithmetic Coding (AC) \& Meteor}
Arithmetic coding (AC) is a form of entropy encoding used in lossless data compression~\cite{5390830}.
A steganographic method that first adopts AC is proposed by~\citet{cryptoeprint:2003/156}. Then, AC is applied in deep generative model and image generation to the field of provably secure steganography~\cite{10.1007/978-3-030-11389-6_5}.
Following these works, in the field of linguistic steganography, researchers have presented a series of variant methods, especially including the original AC-based steganography~\cite{ziegler-etal-2019-neural}, the AC-based method with a self-adjusting mechanism~\cite{shen-etal-2020-near}, and Meteor~\cite{10.1145/3460120.3484550}.

The sort of these AC-based steganographic methods commonly encounter a problem, that is, the precision limitation results in distortion in the original probability distribution at each generative step. Specifically, when the original probabilities are encoded into binary-based intervals, the selected probability for each token always has the form of $2^{-p^{(t)}}$, where $p^{(t)}$ is the precision at time $t$. 
It is almost impossible to maintain the original distribution perfectly, thus introducing distortion.

To mitigate this distortion, one method is using a higher initial precision, and even Python’s \texttt{decimal} module can be used here to support a precision that is higher than the precision of the tensor. 
However, during the iteration of AC-based steganography, the external interval is narrowed and expanded many times, and when the external interval is small, $p^{(t)}$ is also small. Therefore, as the precision is AC-based methods are changed all the times and cannot be controlled well, distortion on probability distribution is inevitable.

Besides, even though Meteor addresses some problems that basic AC-based steganography suffers, Meteor suffers from limited embedding capacity. The reason is that, as Meteor does not narrow the interval successively and only considers each generated symbol separately, thus it cannot fully utilize the entropy. And Meteor does not address the probability-distortion problem that arises in AC-based methods.

\subsection{Adaptive Dynamic Grouping (ADG)}
\citet{zhang-etal-2021-provably} proposed a grouping-based steganographic method called adaptive dynamic grouping (ADG).
At each time step, it dynamically groups the probability distribution of all tokens of the vocabulary into $2^r$ groups with approximately the same probability sum, and then numbers them $0,1,\ldots,2^r-1$. All tokens in each
group represent the same message bits of length $r$. Then, they match the first $r$ bits from the message to be embedded and converts them to a decimal number in $\{0,1,\ldots,2^r-1\}$, and performs random sampling from the normalized distribution of its corresponding group to obtain the next token.
In their assumptions, ADG can theoretically achieve perfect security (no probability distortion) if and only if the grouping is perfectly balanced.

However, the problem is that since the vocabulary-size probability distribution is discrete, the requirement is almost impossible to satisfy. In most cases, the actual distribution used to embed the message is a modified distribution, which is different from the original distribution.

\subsection{Iterative Minimum Entropy Coupling (iMEC)}
\citet{witt2023perfectly} analyzed information-theoretic steganography through the lens of \textit{minimum entropy coupling}. They investigated how much information about a fixed-length secret message can be inferred by the sender and receiver through the selection of tokens, aiming to maximize and accumulate this information until the entire message is determined.
They demonstrated that achieving perfect steganographic security is equivalent to solving a coupling problem, and that maximizing transmission efficiency under perfect security corresponds to solving a minimum entropy coupling problem.
Their proposed iMEC scheme fully exploits the theoretical properties of coupling and minimum entropy coupling. As a result, the method preserves the original probability distribution and achieves provably perfect security.

However, iMEC does have a certain bit error rate. In addition, to achieve minimum entropy coupling and enhance the embedding rate, a considerable amount of computational complexity, specifically $O(|\mathcal{V}| \log |\mathcal{V}|)$, is necessary to couple the probabilities. Low computation efficiency of iMEC makes it difficult to be practically utilized in a vocabulary-size situation.

\subsection{Distribution Copies (Discop)}
\citet{10179287} proposed a provably secure steganographic method based on \textit{distribution copies} (Discop). In this method, several distribution copies are generated by rotating all intervals by specific displacements. At each time step, the message determines which distribution copy to sample from.
Discop also employs an iterative method based on the Huffman tree to further enhance the capacity. Experimental results demonstrated a high utilization rate of entropy. However, the complexity of creating a Huffman tree ($O(|\mathcal{V}|)$ complexity) at each step could not be efficient when a vocabulary-size candidate pool is encoded.

Discop works by creating several rotated \textit{copies} of the model's probability distribution and picking one copy according to the secret bits. This design must let the receiver uniquely tell which copy is used, without ever changing the model's original probabilities. In practice, those rotated copies often overlap on the same token: if the random number falls into an overlap (a \textit{disputed range}), the receiver cannot be sure which copy was chosen. When that happens, the embedding process has to step back and embed fewer bits at that position to keep extraction unambiguous, which directly reduces the per-step payload.
These overlaps become more likely whenever one token is much more probable than the rest, because adding more copies pushes more of the token intervals on top of each other. As a result, the achievable rate is effectively governed by the dominance of the most probable token (not by the full entropy of the distribution), and over long texts it only approaches a stricter ceiling rather than the ideal limit.
The paper itself notes that, empirically, Discop reaches roughly 92\%–95\% of its stated theoretical limit because of these disputed ranges and the need to back off to smaller embeddings when they occur.

\subsection{Sparse Sampling (SparSamp)}
\citet{wang2025sparsampefficientprovablysecure} proposed SparSamp, an efficient and provably secure steganographic method based on sparse sampling. SparSamp embeds messages by combining them with pseudo-random numbers to generate message-derived randomness for sampling. This approach introduces only $O(1)$ additional computational complexity per sampling step, ensuring high computational efficiency.
% However, the entropy utilization of SparSamp is significantly constrained by the length $l$ of the secret message, making it difficult to achieve 100\% entropy utilization for embedding capacity.

SparSamp embeds bits by turning the random numbers used for sampling into \textit{message-driven} numbers. If two different message-driven numbers land on the same token, the receiver cannot tell which message was used. The paper calls this an inherent conflict: longer message chunks raise capacity but also raise the chance of such collisions; shorter chunks reduce collisions but waste headroom. 
To keep extraction unique without changing the distribution, SparSamp deliberately \textit{spreads out} (sparsifies) the allowable random numbers (i.e., it increases the gap between neighboring positions), so different messages are unlikely to pick the same token. That spacing is the price of uniqueness and inevitably leaves a gap to 100\% entropy utilization.
% Empirically the method gets very close (e.g., up to ~99.5% under GPT-2, p=1.0) but still not all the way, and pushing chunk size further runs into numerical limits: with double-precision arithmetic, the safe message length tops out at 1023; going beyond breaks decoding. On top of that, real LLM tokenization brings token ambiguity (multiple valid segmentations). SparSamp’s BackCheck/SynPool machinery can detect and recover from these events, but it consumes small “checkpoint” bits (e.g., ~4 bits per 60) and occasional backtracking—another reason it can’t use every last bit of entropy in practice.

\section{Propositions and Proofs}
\subsection{Proof of Proposition 1}
\label{sec: Proof of Proposition 1}
\textbf{Proposition 1.}
$d_s^{(t)} \sim U(L^{(t-1)}, R^{(t-1)})$.
\begin{proof}
Considering Lines 9--10 in Algorithm~\ref{algorithm: Rotation RC embed}, as $d_s^{(t)} =  L^{(t-1)}+ (d_s^{(t-1)} - L^{(t-1)} + o^{(t)} \times \Delta^{(t-1)}) \bmod \Delta^{(t-1)}$, and $o^{(t)} \in U(0,1)$, we let $A = d_s^{(t-1)} - L^{(t-1)}$ and $B = o^{(t)} \times \Delta^{(t-1)}$. Considering $X = (A+B) \bmod \Delta^{(t-1)}$, for any $x \in [0,\Delta^{(t-1)})$, there is:
\[P(X\leq x) = P((A+B) \bmod \Delta^{(t-1)} \leq x).\]

Let $A = k\Delta^{(t-1)} + r$, where $k \in \mathbb{Z}$ and $r \in [0,\Delta^{(t-1)})$. Considering periodicity of modulo operations, there is
\[(A + B) \bmod  \Delta^{(t-1)} = (r+B) \bmod \Delta^{(t-1)}.\]

\textit{Case 1:} $r+B \leq \Delta^{(t-1)}$. There are $X = r+ B$ and $P(B \leq \Delta^{(t-1)} - r) = \frac{\Delta^{(t-1)}-r}{\Delta^{(t-1)}}$.

\textit{Case 2:} $r+B > \Delta^{(t-1)}$. There are $X = r+ B - \Delta^{(t-1)}$ and $P(B > \Delta^{(t-1)} - r) = \frac{r}{\Delta^{(t-1)}}$.

Therefore, for any $x \in [0,\Delta^{(t-1)})$, there is 
\[P(X \leq x) = \frac{x}{\Delta^{(t-1)}}\]
which means $X \sim U(0,\Delta^{(t-1)})$. 

As $d_s^{(t)} = L^{(t-1)} +X$, $d_s^{(t)} \sim U(L^{(t-1)},L^{(t-1)}+\Delta^{(t-1)}))$, thus $d_s^{(t)} \sim U(L^{(t-1)}, R^{(t-1)})$.
\end{proof}

\subsection{Proof of Proposition 3}
\label{sec: Proof of Proposition 3}
\textbf{Proposition 3.}
\textit{$\Delta^{(t)} \leq 1$ is a sufficient condition for the embedding termination $\frac{L^{(t)} +R^{(t)}}{2} - d_s^{(t)} \in (-0.5,0.5] $ (Lines 14--15 in Algorithm~\ref{algorithm: Rotation RC embed}).}
\begin{proof} If $\Delta^{(t)} = R^{(t)} - L^{(t)} \leq 1$:
\[L^{(t)} \leq d_s^{(t)} < R^{(t)} \leq L^{(t)} + 1\]
% \[\frac{L^{(t)} + R^{(t)}}{2} \in [ \frac{L^{(t)} + L^{(t)}}{2}, \frac{L^{(t)} + L^{(t)} + 1}{2}) \]
\[\frac{L^{(t)} + R^{(t)}}{2} \in (L^{(t)}, L^{(t)}  + 0.5] \subset (L^{(t)}, d_s^{(t)} + 0.5] \]
\[\frac{L^{(t)} + R^{(t)}}{2} \in [R^{(t)}  - 0.5, R^{(t)}) \subset (d_s  - 0.5, R^{(t)}) \]
\[\frac{L^{(t)} + R^{(t)}}{2} \in (d_s  - 0.5, d_s + 0.5] \]
\[\frac{L^{(t)} +R^{(t)}}{2} - d_s^{(t)} \in (-0.5,0.5]. \]
\end{proof}

\subsection{Proof of Computational Security of RRC Steganography}
\label{sec: Proof of Computational Security of RRC Steganography}
\begin{proposition}\label{proposition: computational Security}
For all probabilistic polynomial time (PPT) adversary detection $\mathcal{A}_{\mathcal{D}}$,
\begin{equation*}
    |\text{Pr}[\mathcal{A}_{\mathcal{D}}(x_s) = 1] - \text{Pr}[\mathcal{A}_{\mathcal{D}}(x_c) = 1]| < \text{negl}(\lambda)
\end{equation*}
where $x_s$ is the stegotext, $x_c$ is the normally generated cover text, $\lambda$ is the security parameter of the shared key $K$ (usually the length of $K$), and $\text{negl}(\lambda)$ is a negligible function concerning $\lambda$.
\end{proposition}

\begin{proof}
We prove the computational indistinguishability (security) between stegotext and normal text using a \textbf{hybrid argument}. Define the following hybrid distributions:
\begin{itemize}
    \item $\text{Hyb}_0$ -- real steganographic generation: using the secret message $m_s$ and $K$.
    \item $\text{Hyb}_1$ -- modified steganographic generation: initialized $d_s^{(-1)}$ as a uniform random value over $[0,2^l)$ instead of $m_s$, with other steps unchanged.
    \item $\text{Hyb}_2$ -- normal generation: sample tokens directly from language model $\mathcal{M}$ until the sequence length matches the expected stego length, without interval operations.
\end{itemize}

\paragraph{Step 1: computational indistinguishability (security) between  $\text{Hyb}_0$ and $\text{Hyb}_1$}
\begin{itemize}
    \item In $\text{Hyb}_0$, $d_s^{(-1)} = \text{bin2dec}(m_s)$.
    \item In $\text{Hyb}_1$, $d_s^{(-1)} \sim \text{Uniform}(0,2^l)$.
\end{itemize}

At $t = 0$, the rotation operation updates $d_s^{(0)}$:
\begin{equation*}
    d_s^{(0)} = (d_s^{(-1)} + o^{(0)} \times 2^{l})  \bmod 2^l.
\end{equation*}
Since $o^{(0)} \sim U(0,1)$ implies $o^{(0)} \times 2^l \sim \text{Uniform}(0,2^l)$, we have:
\begin{itemize}
    \item If $d_s^{(-1)}$ is fixed, $d_s^{(0)} \sim \text{Uniform}(0,2^l)$.
    \item If $d_s^{(-1)}$ is uniform, $d_s^{(0)} \sim \text{Uniform}(0,2^l)$ still holds.
\end{itemize}

For $t \geq 1$, $d_s^{(t)} \sim \text{Uniform}(L^{(t-1)}, R^{(t-1)})$ which is independent of initialization. Thus, the token sequence distributions of $\text{Hyb}_0$ and $\text{Hyb}_1$ are  identical (statistically indistinguishable). For any PPT distinguisher $\mathcal{A}_{\mathcal{D}}$:
\begin{equation*}
        |\text{Pr}[\mathcal{A}_{\mathcal{D}}(\text{Hyb}_0) = 1] - \text{Pr}[\mathcal{A}_{\mathcal{D}}(\text{Hyb}_1) = 1]| = 0.
\end{equation*}
Initialization differences are eliminated by the first rotation step.

\paragraph{Step 2: computational indistinguishability (security) between  $\text{Hyb}_1$ and $\text{Hyb}_2$} \mbox{} \\

$\text{Hyb}_1$ uses random $d_s^{(-1)}$ and PRNG, but its conditional token distribution matches the normal generation.
The stopping time $T$ is stochastic, but:
\begin{itemize}
    \item Token distributions conditioned on history match normal generation.
    \item Termination depends on internal uniform variables $d_s^{(t)}$ and $o^{(t)}$, which are inaccessible to distinguishers observing only token sequences.
\end{itemize}

Consider $\text{Hyb}_2$ (normal generation): Sample tokens stepwise from $\mathcal{M}$ for length $L = \mathbb{E}[T]$ (polynomially bounded since $T = O(l)$, efficiently sampleable).

For any fixed sequence $S = [s^{(0)},\ldots, s^{(k-1)}]$, its probability in $\text{Hyb}_1$ is: $\text{Pr}_{\text{Hyb}_1}[S] = (\prod_{t=0}^{k-1}p^{(t)}(s^{(t)})) \times \text{Pr}[\text{terminate at \ } k | S] \times \text{Pr}[\text{no early termination}| S]$.

In $\text{Hyb}_2$ (fixed length $k$): $\text{Pr}_{\text{Hyb}_2}[S] = \prod_{t=0}^{k-1}p^{(t)}(s^{(t)})$.

Though an extra factor $\text{Pr}[\text{terminate} | S]$ exists:
\begin{itemize}
    \item This factor depends on interval lengths determined by $S$, which is bounded in $[0,1]$, and introduces no bias due to language model smoothness.
    \item In computational settings, distinguishers cannot efficiently compute $\text{Pr}[\text{terminate} | S]$ (requires internal state or LM details), and the factor does not alter core conditional distributions.
\end{itemize}

Crucially, the cryptographically secure PRNG ensures its outputs $o^{(t)}$ are computationally indistinguishable from true randomness. If a valid PPT distinguisher $\mathcal{A}_{\mathcal{D}}$ exists for $\text{Hyb}_1$ and $\text{Hyb}_2$, we build an adversary $\mathcal{A}^{\prime}$ to break PRNG security:
\begin{itemize}
    \item $\mathcal{A}^{\prime}$ simulates stego generation but uses a challenge randomness source (PRNG or true random) for $o^{(t)}$.
    \item if $\mathcal{A}_{\mathcal{D}}$ succeeds, $\mathcal{A}^{\prime}$ breaks PRNG security, which leads to a contradiction against the cryptographically secure PRNG.
\end{itemize}
Thus:
\begin{align*}
       & |\text{Pr}[\mathcal{A}_{\mathcal{D}}(\text{Hyb}_1) = 1] - \text{Pr}[\mathcal{A}_{\mathcal{D}}(\text{Hyb}_2) = 1]| \\ & \leq \text{negl}_{\text{PRNG}}(\lambda).
\end{align*}

\paragraph{Step 3: computational indistinguishability (security) between  $\text{Hyb}_0$ and $\text{Hyb}_2$}  \mbox{} \\
By the hybrid argument:
\begin{align*}
       & |\text{Pr}[\mathcal{A}_{\mathcal{D}}(\text{Hyb}_0) = 1] - \text{Pr}[\mathcal{A}_{\mathcal{D}}(\text{Hyb}_2) = 1]| 
       \\ & \leq |\text{Pr}[\mathcal{A}_{\mathcal{D}}(\text{Hyb}_0) = 1] - \text{Pr}[\mathcal{A}_{\mathcal{D}}(\text{Hyb}_1) = 1]| 
       \\ & + |\text{Pr}[\mathcal{A}_{\mathcal{D}}(\text{Hyb}_1) = 1] - \text{Pr}[\mathcal{A}_{\mathcal{D}}(\text{Hyb}_2) = 1]| 
       \\ & \leq 0 + \text{negl}_{\text{PRNG}}(\lambda) = \text{negl}_{\text{PRNG}}(\lambda).
\end{align*}

As $\text{Hyb}_0$ is real steganographic generation, $\text{Hyb}_2$ is normal generation (length-matched), for any PPT distinguisher $\mathcal{A}_{\mathcal{D}}$, there is:
\begin{equation*}
    |\text{Pr}[\mathcal{A}_{\mathcal{D}}(x_s) = 1] - \text{Pr}[\mathcal{A}_{\mathcal{D}}(x_c) = 1]| < \text{negl}(\lambda)
\end{equation*}
where $\text{negl}(\lambda) = \text{negl}_{\text{PRNG}}(\lambda)$, which is negligible.

\end{proof}

\begin{table*}[!t]
% \small
\renewcommand{\arraystretch}{1.0}
\centering
\scalebox{0.9}{
\begin{tabular}{l|ccccc}
\toprule[1.0pt]
Method          & \begin{tabular}[c]{@{}c@{}} Avg / Max KLD \\ (bits/token) \end{tabular} $\downarrow$ & \begin{tabular}[c]{@{}c@{}}Capacity  \\  (bits/token)\end{tabular} $\uparrow$ & \begin{tabular}[c]{@{}c@{}}Entropy \\  (bits/token)\end{tabular} & \begin{tabular}[c]{@{}c@{}}Utilization \\ (\%)  \end{tabular}  $\uparrow$ & \begin{tabular}[c]{@{}c@{}}Speed  \\ (bits/s) \end{tabular} $\uparrow$ \\ \hline \rowcolor{gray!20}
Multinomial sampling &   0 / 0      &     N/A    &  4.59   &     N/A   &     N/A    \\
\hline
AC              &   1.85E-03 / 1.13E-02  &   \underline{4.64}   & 4.65   & \underline{99.81}  &  352.09   \\
ADG            &  1.38E-04 / 1.61E-03    &  3.45   &  4.64  &  74.20  & 25.29    \\
Meteor w/o sort   &  2.80E-02 / 8.34E-02   &   3.13   & 4.54   & 69.03   &  410.79  \\
Meteor w/ sort &    2.77E-02 / 8.12E-02     & 3.65  & 4.52   &  80.76 & 46.02  \\
\hline
iMEC    &   \textbf{0 / 0} &    3.24       &     4.61     &  70.24  &  19.78     \\
Discop w/o sort  &   \textbf{0 / 0}      &   1.92       & 4.67  &  41.08  &  154.25  \\
Discop w/ sort &    \textbf{0 / 0}         &   4.39       & 4.63  & 94.71  &  31.94   \\
SparSamp    &   \textbf{0 / 0}      &    4.35   &  4.53        &  96.08  &   \textbf{852.36}       \\
\hline \rowcolor{orange!20} 
RRC steganography (ours)      &   \textbf{0 / 0}  &  \textbf{4.70}    &  4.67        &  \textbf{100.67} & \underline{750.41}         \\
% \textbf{RC} (Tensor)       &   \textbf{0 / 0}   &  4.55    &  4.72   &  96.41   &  \underline{695.50}       \\
\bottomrule[1.0pt]
\end{tabular}}
 \caption{Quantitative comparison with previous steganographic methods on OPT-1.3b.}
 \label{table: OPT-1.3b_results}
\end{table*}

\begin{table*}[!t]
% \small
\renewcommand{\arraystretch}{1.0}
\centering
\scalebox{0.9}{
\begin{tabular}{l|ccccc}
\toprule[1.0pt]
Method          & \begin{tabular}[c]{@{}c@{}} Avg / Max KLD \\ (bits/token) \end{tabular} $\downarrow$ & \begin{tabular}[c]{@{}c@{}}Capacity  \\  (bits/token)\end{tabular} $\uparrow$ & \begin{tabular}[c]{@{}c@{}}Entropy \\  (bits/token)\end{tabular} & \begin{tabular}[c]{@{}c@{}}Utilization \\ (\%)  \end{tabular}  $\uparrow$ & \begin{tabular}[c]{@{}c@{}}Speed  \\ (bits/s) \end{tabular} $\uparrow$ \\ \hline \rowcolor{gray!20}
Multinomial sampling &   0 / 0      &     N/A    & 3.46    &     N/A   &     N/A    \\
\hline
AC            &      6.92E-04 / 9.90E-03   & \underline{3.53}    &  3.52       &  \underline{100.33}  &   104.37  \\
ADG            &  1.81E-04 / 3.90E-03   &    2.41   & 3.54  & 68.14   &  21.15  \\
Meteor w/o sort   & 1.24E-02 / 4.00E-02    &   2.42     & 3.50   &  69.14  &  98.71  \\
Meteor w/ sort &  1.21E-02 / 4.19E-02   &   2.89  &  3.53  & 81.84  &   50.26  \\
\hline
iMEC    &   \textbf{0 / 0} &    2.48       &   3.43        &  72.35  &    10.30   \\
Discop w/o sort  &   \textbf{0 / 0}      &    1.50      & 3.49  &  42.99 & 127.61   \\
Discop w/ sort &    \textbf{0 / 0}         &  3.33        &  3.48  & 95.72   &  26.13  \\
SparSamp    &   \textbf{0 / 0}      &   3.38   &   3.44    & 98.12   &  \textbf{326.12}        \\
\hline \rowcolor{orange!20} 
RRC steganography (ours)     &   \textbf{0 / 0}   & \textbf{3.57} & 3.52  & \textbf{101.41} & \underline{146.24}          \\
% \textbf{RC} (Tensor)       &   \textbf{0 / 0}    & 3.51    & 3.57          & 98.35   &   \underline{152.63}      \\
\bottomrule[1.0pt]
\end{tabular}}
 \caption{Quantitative comparison with previous steganographic methods on Llama-2-7b.}
 \label{table: Llama-2-7b_results}
\end{table*}

\section{Experimental Details \& Supplementary Information}

\subsection{Supplementary Main Results}
\label{sec: Supplementary Main Results}
Tables~\ref{table: OPT-1.3b_results} and~\ref{table: Llama-2-7b_results} report the average results across various metrics for OPT-1.3b and Llama-2-7b. The findings are consistent with those in Table~\ref{table: GPT-2_results} (GPT-2), where our RRC steganography achieves the highest embedding capacity and embedding utilization among all methods. In contrast, for the two larger language models, our method attains the second-fastest speed, behind SparSamp, which is particularly optimized for efficiency.

In addition, some utilization results exceeding 100\% appear due to sampling bias from finite data. Similar instances of over-100\% utilization have also been reported in recent work~\cite{10919130}.

\begin{table*}[!t]
% \small
\renewcommand{\arraystretch}{1.0}
\centering
\scalebox{0.9}{
\begin{tabular}{l|r|r|r|r|r|r|r|r|r}
\toprule[1.0pt]
Message length (bits)            & 32 & 64 & 128 & 256 & 512 & 1024 & 2048 & 4096 & 8192 \\
\midrule[1.0pt]
Utilization (\%) $\uparrow$    & 100.09   &  99.26  &  \textbf{100.67}   & \underline{100.35}    &   100.30   &   100.02   &   100.14  & 100.14  & 100.08  \\
Speed (bits/s)  $\uparrow$      &  685.32   &  \underline{745.03}  &   \textbf{750.41}  & 720.93    &  701.11    &  588.33    &   403.20   &  248.71  &  184.42  \\
Running time (s)  &  0.047  & 0.086   &   0.171  &  0.355   &  0.730    &     1.740 & 5.079  & 16.469  &  44.421   \\
\bottomrule[1.0pt]
\end{tabular}}
 \caption{Average results on utilization, speed and running time of RRC steganography across various message lengths $l$ on OPT-1.3b.}
 \label{table: scalability_results_OPT-1.3b}
\end{table*}

\begin{table*}[!t]
% \small
\renewcommand{\arraystretch}{1.0}
\centering
\scalebox{0.9}{
\begin{tabular}{l|r|r|r|r|r|r|r|r|r}
\toprule[1.0pt]
Message length (bits)            & 32 & 64 & 128 & 256 & 512 & 1024 & 2048 & 4096 & 8192 \\
\midrule[1.0pt]
Utilization (\%) $\uparrow$    & \textbf{102.01}   &  100.29  &  \underline{101.41}   &  100.52   & 100.25     &  100.31    &  100.35     & 100.17  & 100.04  \\
Speed (bits/s)  $\uparrow$      &  142.67  &  \underline{144.63} &   \textbf{146.24}   &  142.00   &   126.51   &  92.65    &  79.38   & 63.21  &   46.10 \\
Running time (s)  &  0.224  & 0.443   &  1.143   & 1.803    &   4.047   & 11.052     &   25.800   & 64.800  &   177.701  \\
\bottomrule[1.0pt]
\end{tabular}}
 \caption{Average results on utilization, speed and running time of RRC steganography across various message lengths $l$ on Llama-2-7b.}
 \label{table: scalability_results_Llama-2-7b}
\end{table*}

\begin{table*}[!t]
\renewcommand{\arraystretch}{1.0}
\centering
\scalebox{1.0}{
\begin{tabular}{l|cc|cc|cc}
\toprule[1.0pt]
\multirow{2}{*}{Steganalysis model} & \multicolumn{2}{c|}{GPT-2} & \multicolumn{2}{c|}{OPT-1.3b} & \multicolumn{2}{c}{Llama-2-7b} \\ \cline{2-7}
                                    & RRC (ours)       & SparSamp      & RRC (ours)       & SparSamp        & RRC (ours)        & SparSamp         \\ \midrule[1.0pt]
bert-base-uncased    &     49.3\%      &      49.9\%         &    48.1\%        &      48.6\%           &   49.7\%          &      50.9\%            \\
roberta-base       &      51.1\%     &      50.5\%         &   50.1\%         &      49.7\%           &         51.3\%    &        50.5\%          \\
roberta-large     &     48.9\%      &      49.2\%         &    52.6\%        &        51.8\%         &    47.7\%         &  47.6\%    \\ \bottomrule[1.0pt]           
\end{tabular}}

\caption{Comparison of steganalysis accuracies against RRC steganography and SparSamp steganography under cases where models for steganography vary and models for steganalysis vary.}
\label{table: Comparison of Steganalysis}
\end{table*}

\begin{table*}[!t]
\renewcommand{\arraystretch}{1.0}
\centering
\scalebox{1.0}{
\begin{tabular}{l|ccc}
\toprule[1.0pt]
Method                   & GPT-2 & OPT-1.3b & Llama-2-7b \\ \hline \rowcolor{gray!20}
Multinomial sampling     &   81.05    &    74.09      &     74.99       \\ \hline
SparSamp                 &   83.56    &     73.18     &     73.37        \\  \rowcolor{orange!20} 
RRC steganography (Ours) &  82.79     &    72.66      &     75.24       \\ \bottomrule[1.0pt]
\end{tabular}}
\caption{Comparison of median perplexity between RRC steganography and SparSamp steganography under cases where models for steganography vary.}
\label{table: Comparison of perplexity}

\end{table*}

\begin{table*}[!t]
% \small
\renewcommand{\arraystretch}{1.0}
\centering
\scalebox{0.9}{
\begin{tabular}{l|ccccc}
\toprule[1.0pt]
Method          & \begin{tabular}[c]{@{}c@{}} Avg / Max KLD \\ (bits/token) \end{tabular} $\downarrow$ & \begin{tabular}[c]{@{}c@{}}Capacity  \\  (bits/token)\end{tabular} $\uparrow$ & \begin{tabular}[c]{@{}c@{}}Entropy \\  (bits/token)\end{tabular} & \begin{tabular}[c]{@{}c@{}}Utilization \\ (\%)  \end{tabular}  $\uparrow$ & \begin{tabular}[c]{@{}c@{}}Speed  \\ (bits/s) \end{tabular} $\uparrow$ \\  \hline \rowcolor{blue!10}
\multicolumn{6}{c}{Implemented in GPT-2} \\ \hline  
Vanilla RC steganography &   1.67E-14 / 5.73E-12  & 5.94  &    5.93   &   100.20 &   1605.31      \\
RRC steganography    &   0 / 0  & 5.93  &    5.93   &   99.98 &   1554.66      \\ \hline \rowcolor{blue!10}
\multicolumn{6}{c}{Implemented in OPT-1.3b} \\ \hline  
Vanilla RC steganography &   1.46E-14 / 4.13E-12  & 4.70 &  4.69    &  100.19  &   770.49      \\
RRC steganography    &   0 / 0  & 4.70  &    4.67   &   100.67 &   750.41      \\ \hline \rowcolor{blue!10}
\multicolumn{6}{c}{Implemented in Llama-2-7b} \\ \hline  
Vanilla RC steganography &   7.26E-13 / 1.99E-12  &   3.55    & 3.55  & 99.91 &  153.64      \\
RRC steganography    &   0 / 0  & 3.57  &    3.52   &   101.41 &   146.24      \\
% \textbf{RC} (Tensor)       &   \textbf{0 / 0}      &  5.58     &    5.91   &  94.40                                                             &    \textbf{1519.64}     \\
\bottomrule[1.0pt]
\end{tabular}}
 \caption{Comparison between vanilla RC steganography and RRC steganography in various metrics.}
 \label{table: Ablation}
\end{table*}

\subsection{Supplementary Results of Scalability}
\label{sec: Supplementary Results of Scalability}
Tables~\ref{table: scalability_results_OPT-1.3b} and~\ref{table: scalability_results_Llama-2-7b} list the average utilization, speed, and running time when RRC steganography embeds the secret message with various lengths (up to 8192 bits) on OPT-1.3b and Llama-2-7b. Similarly to the results of~\ref{table: scalability_results}, RRC steganography can achieve steady entropy utilization around 100\%.

\subsection{Supplementary Results of Comparison}
\label{sec: Supplementary Results of Comparison}
Table~\ref{table: Comparison of Steganalysis} reports steganalysis accuracies across three steganalysis models, three language models for steganography, and two steganographic methods (our RRC and the current state-of-the-art baseline, SparSamp). The experimental settings for steganalysis follow the details in Appendix~\ref{sec: steganalysis}.

Table~\ref{table: Comparison of perplexity} presents the median perplexities under various language models for steganography and two steganographic methods (with multinomial sampling included for reference). The embedded secret message was a random 128-bit sequence, with 1,000 samples per group. Median values are reported because average perplexities are heavily skewed by extreme outliers.

Overall, both methods yield near-random steganalysis accuracies (about 50\%) and comparable perplexity scores (close to those of multinomial sampling), due to the provable security property shared by RRC steganography and SparSamp. Besides,  we emphasize that lower perplexity does not imply higher imperceptibility, due to the conflict between perceptual and statistical imperceptibility (the Psic Effect) in steganography~\cite{9193914}.

% \usepackage{multirow}

\subsection{Steganalysis}
\label{sec: steganalysis}
We generated 5,000 pairs of cover texts (via multinomial sampling) and steganographic texts, respectively implemented on GPT-2, OPT-1.3b, and Llama-2-7b. In each pair, the lengths (token number) of two texts are the same. The initial contexts for generation are the first 10 words from sequences randomly selected from the C4 dataset. For each experimental group, 5,000 texts are split in a 6:2:2 ratio to create the training, validation, and test sets.

For fine-tuning BERT or RoBERTa models, we use Adam~\cite{kingma2017adammethodstochasticoptimization} as the optimizer with a learning rate of $5 \times 10^{-5}$. The batch size is set to 2048, and the discriminator is trained for 20 epochs, running time of the whole training process is approximately 5 minutes.

\subsection{Ablation on Rotation Mechanism}
\label{sec: Ablation}
To quantify the effect of the rotation mechanism, we implemented a variant without rotation (\textit{vanilla RC}). As shown in Table~\ref{table: Ablation} (where the experimental setups follow Section~\ref{sec: Setup}), we can find that, in vanilla RC steganography, even though it can obtain the higher embedding speed compared to RRC steganography, 0 KL divergence cannot be achieved. Most importantly, theoretical security cannot be achieved.
Therefore, the provable security of vanilla RC steganography cannot be ensured as shown in the analysis in Section~\ref{sec: Security Issues}.

\section{Discussion on Practical Issues}
One practical issue of linguistic steganography is \textbf{tokenization inconsistency}. The stegotext $t_s$ generated by $\mathcal{S}_{\text{emb}}(\mathcal{M}, m_s)$ is essentially a sequence of tokens. Before transmission, the sender must detokenize this sequence using a tokenizer to produce the final stegotext. Consequently, during extraction $\mathcal{S}_{\text{ext}}(\mathcal{M}, t_s)$, any tokenization inconsistency may cause extraction to fail or yield an incorrect secret message. This problem can be avoided only in a few tokenizer-free linguistic steganographic approaches~\cite{math8091558,10831652}. Recently, several disambiguation methods have been proposed to mitigate this issue~\cite{nozaki-murawaki-2022-addressing,yan2023A,10831370,qi2024provably, yan2025addressingtokenizationinconsistencysteganography}. These methods are orthogonal to our proposed RRC steganography and can be compatible, since they operate on candidate pools prior to steganographic processing.

Another practical issue of linguistic steganography is \textbf{hardware indeterminism}~
\cite{atil2025nondeterminismdeterministicllmsettings}, whereby LLM outputs and the probability distribution of the next token given the same context can vary across different hardware settings. Such variability poses serious risks for reliable message extraction in real-world deployments, since even small shifts in probability values may result in incorrect extraction. This issue underscores the importance of explicitly addressing hardware-level robustness as an additional design dimension for practical steganographic systems, alongside imperceptibility and capacity, in future research.

\section{Samples of Texts}
We present examples of stegotexts generated by RRC steganography and non-steganographic texts generated by multinomial sampling. Each stegotext embeds a 128-bit random secret message. 
The initial context is ``\texttt{Occasionally when I get some free time, I'll do}.''
Following the approach of Ziegler et al.~\cite{ziegler-etal-2019-neural}, we terminate the generation process once the proposed method has finished embedding the message. In the following examples, for each language model, the non-steganographic text has the same token number as that of the corresponding stegotext.

\begin{tcolorbox}
[colback=gray!10!white, colframe=gray!80!black, title=\texttt{Texts generated by GPT-2}]
    \small
    \textbf{Stegotext:}
    
    \texttt{Occasionally when I get some free time, I'll do something that uses all of those sensors scanned at the bottom of the computer - search for something. But I don't know how to do that} \\

    \textbf{Non-steganographic text:}
    
    \texttt{Occasionally when I get some free time, I'll do some stretch and the lights rise over me, and will do some real experimenting. I didn't really think about art where}
\end{tcolorbox}

\begin{tcolorbox}
[colback=gray!10!white, colframe=gray!80!black, title=\texttt{Texts generated by OPT-1.3b}]
    \small
    \textbf{Stegotext:}
    
    \texttt{Occasionally when I get some free time, I'll do flat screen, planner style arrangement cards. It really highlights that the cards are supposed to be focused} \\

    \textbf{Non-steganographic text:}

    \texttt{Occasionally when I get some free time, I'll do that in survival. It's loads more fun to just play around and search ways\* to build a nest and be creative in }
\end{tcolorbox}

\begin{tcolorbox}
[colback=gray!10!white, colframe=gray!80!black, title=\texttt{Texts generated by Llama-2-7b}]
    \small
    \textbf{Stegotext:}
    
    \texttt{Occasionally when I get some free time, I'll do a quick Google search on a random topic that interests me (if I have one free not sitting in front of a computer screen!), and just see where my curiosity takes me. The first thing
    } \\

    \textbf{Non-steganographic text:}

    \texttt{Occasionally when I get some free time, I'll do a Google Searche for ""badminton"". It is refreshing not to find the many images that can be found with another more popular global pastime. My experiment of searching finds pictures }
\end{tcolorbox}

\end{document}